\documentclass[nohyperref]{article}

\usepackage{microtype}
\usepackage{graphicx}
\usepackage{subfigure}
\usepackage{booktabs} 

\usepackage{hyperref}

\usepackage[accepted]{icml2023}

\usepackage{amsmath}
\usepackage{amssymb}
\usepackage{mathtools}
\usepackage{amsthm}

\usepackage{color}
\definecolor{darkblue}{rgb}{0,0,0.75}

\hypersetup{colorlinks=true,
            linkcolor=black,
            filecolor=black,
            urlcolor=darkblue,
            citecolor=darkblue, pdfauthor={Pascal Welke, Maximilian Thiessen, Fabian Jogl, and Thomas Gärtner},
            pdfkeywords={graph homomorphism, graph representation, graph kernels},
            bookmarksnumbered=true
            }
         
\usepackage{booktabs}

\usepackage{blkarray}
\usepackage{tikz}
\usetikzlibrary{calc}
            
\usepackage{parskip}
\usepackage{amsthm, amsmath, amssymb}

\usepackage{thmtools}
\usepackage{thm-restate}

\usepackage{xfrac}

\usepackage{commath}

\newtheorem{theorem}{Theorem}
\newtheorem{definition}[theorem]{Definition}
\newtheorem{lemma}[theorem]{Lemma}

\newtheorem{corollary}[theorem]{Corollary}

\newcommand{\demph}[1]{\emph{#1}} 

\newcommand{\dotprod}[2]{\langle #1, #2 \rangle}
\renewcommand{\phi}{\varphi}
\newcommand{\singlehomomorphism}[0]{\Phi}
\DeclareMathOperator{\homo}{hom}
\DeclareMathOperator{\Geom}{Geom}
\DeclareMathOperator{\tw}{tw}
\DeclareMathOperator{\Poi}{Poi}
\DeclareMathOperator{\Uni}{\mathcal{U}}
\newcommand{\G}{\mathcal{G}}
\newcommand{\F}{\mathcal{F}}
\newcommand{\scH}{\mathcal{H}}
\newcommand{\scO}{\mathcal{O}}
\newcommand{\isomorphic}{\simeq}
\newcommand{\naturals}[0]{\mathbb{N}}
\newcommand{\reals}[0]{\mathbb{R}}
\newcommand{\scD}{\mathcal{D}}

\newcommand{\E}{\mathbb{E}}
\newcommand{\Set}[1]{\left\{ #1 \right\}}

\newenvironment{proofsketch}{\proof}{\endproof}

\newcommand{\homGNN}[1]{#1+hom}

\newcommand{\tablecell}[1]{\begin{tabular}[x]{@{}l@{}}#1\end{tabular}  }

\usepackage[capitalize,noabbrev]{cleveref}

\theoremstyle{plain}

\usepackage[textsize=tiny]{todonotes}

\icmltitlerunning{Expectation-Complete Graph Representations with Homomorphisms}

\begin{document}

\twocolumn[
\icmltitle{Expectation-Complete Graph Representations with Homomorphisms}

\icmlsetsymbol{equal}{*}

\begin{icmlauthorlist}
\icmlauthor{Pascal Welke}{equal,bonn,wien}
\icmlauthor{Maximilian Thiessen}{equal,wien}
\icmlauthor{Fabian Jogl}{wien,caiml}
\icmlauthor{Thomas Gärtner}{wien}
\end{icmlauthorlist}

\icmlaffiliation{bonn}{Machine Learning and Artificial Intelligence Lab, University of Bonn, Germany}
\icmlaffiliation{wien}{Research Unit Machine Learning, TU Wien, Austria}
\icmlaffiliation{caiml}{Center for Artificial Intelligence and Machine Learning, TU Wien, Austria}

\icmlcorrespondingauthor{Pascal Welke}{pascal.welke@tuwien.ac.at}
\icmlcorrespondingauthor{Maximilian Thiessen}{maximilian.thiessen@tuwien.ac.at}

\icmlkeywords{Machine Learning, ICML}

\vskip 0.3in
]

\printAffiliationsAndNotice{\icmlEqualContribution} 

\begin{abstract}

We investigate novel random graph embeddings that can be computed in expected polynomial time and that are able to distinguish all non-isomorphic graphs \emph{in expectation}. Previous graph embeddings have limited expressiveness and either cannot distinguish all graphs or cannot be computed efficiently for every graph. To be able to approximate arbitrary functions on graphs, we are interested in efficient alternatives that become arbitrarily expressive with increasing resources.  
Our approach is based on Lovász' characterisation of graph isomorphism through an infinite dimensional vector of homomorphism counts. 
Our empirical evaluation shows competitive results on several benchmark graph learning tasks. \end{abstract}

\section{Introduction}
\label{sec:intro}

We study novel efficient and expressive graph embeddings motivated by Lovász' characterisation of graph isomorphism through homomorphism counts. 
While most graph embeddings drop \emph{completeness}---the ability to distinguish all pairs of non-isomorphic graphs---in favour of runtime, we devise efficient embeddings that retain completeness \emph{in expectation}. 
The specific way in which we sample a fixed number of pattern graphs guarantees an expectation-complete embedding in expected polynomial time.
In this way, repeated sampling will eventually allow us to distinguish all pairs of non-isomorphic graphs, a property that no efficiently computable deterministic embedding  can guarantee. 
In comparison, most recent graph neural networks are inherently limited by the expressiveness of some $k$-dimensional Weisfeiler-Leman isomorphism test \citep{morris2019weisfeiler, xu2018how}.
\vspace{2em}

Our approach to achieve an expectation-complete graph embedding is based on homomorphism counts. 
These are known to determine various properties of graphs important for learning, such as the degree sequence or the eigenspectrum \citep{hoang2020graph}.
Furthermore, homomorphism counts are related to the Weisfeiler-Leman hierarchy \citep{dvovrak2010recognizing, dell2018lov}, which is the standard measure for expressiveness on graphs \citep{morris2019weisfeiler}. 
They also determine subgraph counts \citep{curticapean2017homomorphisms} and the distance induced by the homomophism counts is asymptotically equivalent to the \emph{cut distance}, which \citet{grohe2020Word} and \citet{klopp2019optimal} motivated as an appropriate graph similarity for graph learning tasks.

In \Cref{sec:notation} we introduce the required concepts. In \Cref{sec:main} we discuss that general expectation-complete embeddings can eventually distinguish all pairs of non-isomorphic graphs (\Cref{lemma:repeat_until_injective}), which leads to a universal representation (\Cref{thm:universal}). Then we propose our expectation-complete embedding based on sampling entries from the Lovász vector (\Cref{thm:Lovasz_complete_exp}) and bound the number of samples required to provably get as close as desired to the full Lovász vector (\Cref{thm:lovasz_hoeffding_finite_sample}). 
In \Cref{sec:computation}, we show how to compute our embedding efficiently in expected polynomial time (\Cref{thm:completeRandomPolynomial}). In \Cref{sec:practical}, we show how to combine our embedding with graph neural networks. Finally, we discuss related work in \Cref{sec:related} and show competitive results on benchmark datasets in \Cref{sec:experiments} before \Cref{sec:conclusion} concludes. \section{Background and Notation}
\label{sec:notation}

We start by defining the required concepts and notation.
A \demph{graph} $G = (V,E)$ consists of a set $V=V(G)$ of \demph{vertices} and a set $E=E(G) \subseteq \{e \subseteq V \mid |e|=2\}$ of \emph{edges}. In this work we only consider undirected graphs. 
The \emph{size} $v(G)$ of a graph $G$ is the number of its vertices and by $\G_n$ we denote the set of all graphs with size at most $n\in\naturals$.
In the following $F$ and $G$ denote graphs, where $F$ represents a \emph{pattern} graph and $G$ a graph in our training set.
A \demph{homomorphism} $\singlehomomorphism: V(F) \to V(G)$ is a map that preserves edges, i.e. $\{v,w\} \in E(F) \Rightarrow  \{\singlehomomorphism(v), \singlehomomorphism(w)\} \in E(G)$.
Note that homomorphisms, unlike \emph{subgraph isomorphisms}, allow non-injectivity: multiple vertices of $F$ can be mapped to the same vertex of $G$, see \Cref{fig:homomorphism_example}.  Let $\homo(F,G)$ denote the number of homomorphisms from $F$ to $G$ and let $\phi_\G(G)= (\homo(F,G))_{F\in\G}$ denote the vector of homomorphism counts from each graph of a family of graphs $\G$ to $G$. We define the shorthand $\phi_n(G)=\phi_{\G_n}$(G). We also define the $\demph{homomorphism density}$ $t(F,G) = \homo(F,G)/v(G)^{v(F)}$, corresponding to the probability that a mapping from $V(F)$ to $V(G)$ drawn uniformly at random is a homomorphism. Similarly to $\phi$, we define $\psi_\G(G)=t(\G,G) = (t(F,G))_{F\in\G}$ and $\psi_n(G)=\psi_{\G_n}$.
An \demph{isomorphism} between two graphs $G$ and $G'$ is a bijection $I:V(G)\rightarrow V(G')$ such that $\{v,w\}\in E(G)$ if and only if \{$I(v),I(w)\}\in E(G')$. If there is an isomorphism between $G$ and $G'$, we say they are \demph{isomorphic} and denote it as $G\isomorphic G'$. We say that a probability distribution $\scD$ over a countable domain $\mathcal{X}$ has \emph{full support} if each $x\in \mathcal{X}$ has nonzero probability $\Pr_{X\sim \scD}(X=x)>0$.
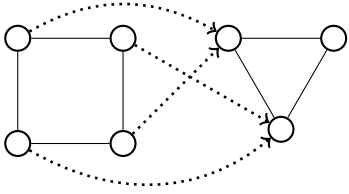
\begin{figure}[t!]
    \centering
    \begin{tikzpicture}[scale=.7]
        \begin{scope}[every node/.style={circle,thick,draw}]
\node (T1) at (4,4) {};
            \node (T2) at (6,4) {};
            \node (T3) at (5, 4-1.73) {};
            \draw (T1) -- (T2);
            \draw (T2) -- (T3);
            \draw (T1) -- (T3);
\node (A1) at (0,4) {};
            \node (A2) at (2,4) {};
            \node (A3) at (2,2) {};
            \node (A4) at (0,2) {};
            \draw (A1) -- (A2);
            \draw (A3) -- (A4);
            \draw (A1) -- (A4);
            \draw (A3) -- (A2);
            \draw [->, dotted, line width=0.35mm] (A1) to [out=30,in=150] (T1);
            \draw [->, dotted, line width=0.35mm] (A4) to [out=330,in=220] (T3);
            \draw [->, dotted, line width=0.35mm] (A3) to  (T1);
            \draw [->, dotted, line width=0.35mm] (A2) to (T3);
        \end{scope}
    \end{tikzpicture}
    \caption{Example homomorphism: mapping a 4-cycle to an edge.}
    \label{fig:homomorphism_example}
\end{figure}
 
\subsection{Complete Graph Embeddings}
\label{sec:methods}

Classical graph kernel and recent (neural) graph representation methods perform learning on graphs by (potentially implicitly) embedding them into a real vector space $\scH$. A \demph{graph embedding} is a map $\phi:\G\rightarrow \scH$ defined on a set of graphs $\G$. A graph embedding $\phi$ is called \demph{permutation-invariant} if for all  $G \isomorphic G' \in \G$ it holds that $\phi(G)=\phi(G')$. All common graph kernels \citep{kriege2020survey} and standard message-passing neural networks \citep{xu2018how} are permutation-invariant.  Now we define \demph{completeness}, which requires the opposite direction of the implication.
\begin{definition}
A permutation-invariant graph embedding $\phi:\G\rightarrow V$ is complete (on $\G$) if $\phi(G)\neq \phi(G')$ for all non-isomorphic $G,G'\in\G$.
\end{definition}
Completeness is necessary if we want to be \demph{universal}, that is, be able to approximate any permutation-invariant function $f:\G\rightarrow \reals$. In particular we would not be able to approximate a function $f$ with $f(G)\neq f(G')$ for two non-isomorphic graphs $G$ and $G'$ with $\phi(G)=\phi(G')$.

Complete graph embeddings allow to determine whether two graphs are isomorphic, as $G\simeq G'$ if and only if $\phi(G)=\phi(G')$. Deciding graph isomorphism is a classical problem in graph theory whose computational complexity is a major open problem \citep{babai2016graph}. While the problem is in NP, neither a  polynomial-time algorithm is known nor it is known whether the problem is NP-complete. Thus, we always face a trade-off between efficiency and expressiveness: complete graph embeddings are unlikely to be computable in polynomial time  \citep{gaertner2003kernel} and hence most graph representations drop completeness in favour of polynomial runtime.

If $\scH$ is a (real) Hilbert space with inner product $\dotprod{\cdot}{\cdot}\rightarrow\reals$, and not just a vector space, we can define a \demph{graph kernel} $k_\phi(G,G')= \dotprod{\phi(G)}{\phi(G')}$ using any permutation-invariant graph embedding $\phi:\G\rightarrow\scH$. We call $k_\phi$ \demph{complete} if $\phi$ is complete. Note that \begin{align}\label{eq:kernel_distance}
k_\phi(G,G)-2k_\phi(G,G')+k_\phi(G',G')=\|\phi(G)-\phi(G')\|^2
\end{align} which for a complete kernel is 0 if and only if $G\simeq G'$. Thus, evaluating a complete graph kernel is at least as hard as deciding graph isomorphism, even if $\phi$ is not known or computed explicitly \citep{gaertner2003kernel}. 

In this work, we avoid the previously mentioned trade-off by using random graph embeddings than can be computed in expected polynomial time. While dropping completeness, this allows us to keep a slightly weaker yet still desirable property: \emph{completeness in expectation}.

\section{Expectation-Complete Graph Embeddings}
\label{sec:main}

In the remainder of this work we will consider \emph{random graph embeddings}. These are graph embeddings $\phi_X:\G\rightarrow\scH$ that are parameterised by a random variable $X$. 
Algorithmically, we can think of $\phi_X(G)$ as first sampling a random variable $X\sim \scD$ from a distribution $\scD$ and then computing $\phi_X(G)$. 
If the expectation $\E_{X\sim\scD}[\phi_X(G)]$ is defined for all $G \in \G$, we can define a (deterministic) graph embedding $\E_{X\sim\scD}[\phi_X(\cdot)]:\G\rightarrow\scH$. 
This leads us to the central notion of this paper.
\begin{definition}
A random graph embedding $\phi_X$ is expectation-complete if the graph embedding $\E_X[\phi_X(\cdot)]$ is complete.
The corresponding kernel $k_X(G,G') = \langle \phi_X(G),\phi_X(G')\rangle$ is expectation-complete if $\phi_X$ is expectation-complete.
\end{definition}

Expectation-complete graph embeddings satisfy a useful property, which no non-complete deterministic graph embedding can satisfy: they eventually will be complete  if we sample often enough.

\begin{lemma}\label{lemma:repeat_until_injective}
Let $\phi_X:\G\rightarrow\scH$ be a expectation-complete graph embedding and $G,G'\in\G$ which are not isomorphic. For any $\delta>0$, there exists $L\in\naturals$ such that for all $\ell\geq L$ $$(\phi_{X_1}(G),\dots,\phi_{X_\ell}(G)) \neq (\phi_{X_1}(G'),\dots,\phi_{X_\ell}(G'))$$ with probability $1-\delta$, where $X_1,\dots,X_\ell\sim\scD$ i.i.d. 
\end{lemma}

\begin{proof}
Let $G,G'$ be non-isomorphic graphs. Since $\phi_X$ is expectation-complete, it must hold that $\E[\phi_X(G)]\neq \E[\phi_X(G')]$, which in particular means that there exists a set $A_{G,G'}$ of outcomes of $X$ with $\Pr(X\in A_{G,G'})=p_{G,G'}>0$ such that for all $a\in A_{G,G'}$ it holds that $\phi_a(G)\neq\phi_a(G')$. We need $\Pr(\exists i\in\{1,\dots,\ell\}: X_i\in A_{G,G'})\geq 1-\delta$, hence $1-(1-p_{G,G'})^\ell \geq 1-\delta$ must hold. Solving for $\ell$ we see that $\ell\geq L=\left\lceil\frac{\log(\sfrac{1}{\delta})}{\log(\frac{1}{1-p_{G,G'}})}\right\rceil$ is sufficient to guarantee that there will be at least one $X_i$ in $A$ with probability at least $1-\delta$, implying $\phi_{X_i}(G)\neq\phi_{X_i}(G')$.
\end{proof}

This leads to the following result, that sampling eventually yields universality.

\begin{theorem}
\label{thm:universal}
    Let $n\in\naturals$, $\phi_X:\G_n\rightarrow\reals^d$ be a finite-dimensional expectation-complete graph embedding and $f: \G_n \rightarrow \reals$ a permutation-invariant function. For any $\varepsilon>0$ and $\delta>0$ there exists an $\ell\in\naturals$ and a multi-layer-perceptron $g:\reals^{d \ell} \rightarrow \reals$ such that $$|f(G)-g(\phi_{X_1}(G),\dots,\phi_{X_\ell}(G))|<\varepsilon$$ for all $G\in\G_n$ with probability at least $1-\delta$, where $X_1,\dots,X_\ell\sim\scD$ i.i.d. 
\end{theorem}

\begin{proof}
    Let $N=|\G_n|$, $\G_n=\{G_1,\dots,G_N\}$ and $f(G_i)=y_i$ for all $i$. As in the proof of $\Cref{lemma:repeat_until_injective}$ we know that for each pair $G,G'\in\G$ of non-isomorphic graphs there exists an event $A_{G,G'}$ with non-zero probability $p_{G,G'}$ guaranteeing that $\phi_X(G)\neq\phi_X(G')$. Let $p=\min_{G,G'} p_{G,G'}>0$. We have to satisfy this for all pairs of non-isomorphic graphs simultaneously. By applying a union bound on the complement (meaning at least one $A_{G,G'}$ does not happen) and bounding each of these terms through \Cref{lemma:repeat_until_injective}, we see that $$\ell\geq \frac{\log|\G_n|+\log(1/\delta)}{\log(\frac{1}{1-p})}$$ samples are sufficient to guarantee that the embedding $\phi_\ell(G)=(\phi_{X_1}(G),\dots,\phi_{X_\ell}(G))$ is complete with probability $1-\delta$. Note that $\log|\G_n|\leq n^2$, meaning that if we treat $p$ and $\delta$ as constants then $\scO(n^2)$ many samples suffice. 
    
    It remains to show that there is an MLP $g$ which can approximate the points $(\phi_\ell(G_1),y_1),\dots,(\phi_\ell(G_N),y_N)$. It is clear that there exists a multivariate polynomial exactly fitting all the points. Then we can apply universal function approximation to the bounded region spanned by the $N$ points and approximate the polynomial.
\end{proof}

\subsection{Expectation-Completeness Through Graph Homomorphisms}

We now present one way to achieve expectation-completeness. We use the classical result of \citet{lovasz1968vector} that all homomorphism counts up to $n=\max\{v(G),v(G')\}$ determine if $G$ and $G'$ are isomorphic.

\begin{theorem}[\citet{lovasz1968vector}\protect\footnote{For a more recent proof see  Theorem 5.29 and the comments below in \citet{lovasz2012book}.}]\label{thm:Lovászvector}
Two graphs $G,G'\in\G_n$ are isomorphic if and only if $\phi_n(G)=\phi_n(G')$.
\end{theorem}

This provides a powerful graph embedding for learning tasks on graphs \citep{dell2018lov, hoang2020graph, barceloGraphNeuralNetworks2021}.
We can define a simple kernel on $\G_n$ with the canonical inner product using $\phi_n$.
\begin{definition}[Complete Lovász kernel]
Let $k_{\phi_n}(G,G') = \dotprod{\phi_n(G)}{\phi_n(G')}$.
\end{definition}
Note that $\phi_n$ and $k_{\phi_n}$ are both complete on $\G_n$, and hence can be used to distinguish non-isomorphic graphs of size up to $n$.
 We can use the Lovász vector embedding $\phi_n$ to devise graph embeddings that are expectation-complete. For that let $e_F\in\mathbb{R}^{\G_n}$ be the `$F$th' standard basis unit-vector of $\reals^{\G_n}$. For a distribution $\scD$ with full support on $\G_n$ define the graph embedding $\phi_F(G) = \hom(F,G)e_F$ with $F\sim\scD$. 
\begin{theorem}\label{thm:Lovasz_complete_exp}
For a distribution $\scD$ with full support on $\G_n$ and $F\sim \scD$, the random embedding $\phi_F(\cdot)$ and the corresponding kernel are expectation-complete on $\G_n$.
\end{theorem}
\begin{proof}
Let $\phi_F$ with $F\sim \scD$ be as stated and $G\in\G_n$. Then 
$$g=\E_{F}[\phi_F(G)]=\sum_{F'\in\G_n} \Pr\left(F=F'\right)\homo(F',G)e_{F'}\,.$$
The vector $g$ has the entries $(g)_{F'}=\Pr\left(F=F'\right)\homo(F',G)$. Let $G'$ be a graph that is non-isomorphic to $G$ and let $g'=\E_{F}[\phi_F(G')]$ accordingly. By Theorem \ref{thm:Lovászvector} we know that $\phi_n(G)\neq \phi_n(G')$. Thus, there is an $F'$ such that $\homo(F',G)\neq \homo(F',G')$. By definition of $\scD$ we have that $\Pr(F=F')> 0$ and hence $\Pr(F=F')\homo(F',G)\neq \Pr(F=F')\homo(F',G')$ which implies $g\neq g'$. That shows that $\E_{F}[\phi_F(\cdot)]$ is complete and concludes the proof.
\end{proof}

We now analyse how close we are  to the actual Lovász kernel, if we sample $\ell$ patterns $\F = (F_1, \ldots, F_\ell)$ i.i.d. from $\scD$. We consider $\phi_\F = \sum_{F \in \F} \phi_F$ and the kernel $k_\F (G, G') = \langle \phi_\F (G), \phi_\F (G') \rangle$.
While formally working in $\reals^{\G_n}$, we can restrict the analysis (and practical computation) to $\reals^{\F}$, ignoring dimensions that only contain zeros.

We apply standard techniques similar to \citet{rahimi2007random}, \citet{kontorovich2009universal}, \citet{shervashidze2009efficient}, and \citet{wu2019scalable}. For convenience we will perform the analysis using the homomorphism densities $\psi_F$. Let $D\in\reals^{\G_n\times\G_n}$ be a diagional matrix with $D_{FF}=\Pr_{X\sim\scD}(X=F)$ and let $J_F\in\{0,1\}^{\G_n\times\G_n}$ be a matrix that is $1$ at the $FF$th position and $0$ everywhere else. For the expectation of the random kernel $\dotprod{\psi_F(G)}{\psi_F(G')}$ it holds that
\begin{align*}
    \E_{F\sim \scD}[\dotprod{\psi_F(G)}{\psi_F(G')}]&=\E_{F\sim \scD}[\psi^\mathsf{T}_{\G_n}(G) J_F\psi_{\G_n}(G')]\\
    &=\dotprod{\sqrt{D}\psi_{\G_n}(G)}{\sqrt{D}\psi_{\G_n}(G')}
    \\&=: k_\scD(G,G')\,.
\end{align*}
Note that $k_\scD(G,G')$ is still a complete kernel as the complete graph embedding $\psi_{\G_n}$ is just scaled by $\sqrt{D}$, which is invertible as $\scD$ has full support. 
For a sample $\F$ of $\ell$ patterns we get the joint (averaged) embedding $\psi_\F(G)=\sfrac{1}{\sqrt{\ell}}(t(F_1,G),\dots,t(F_\ell,G))$ and get the corresponding (averaged) kernel
$$\Tilde{k}_\F(G,G')=\dotprod{\psi_\F(G)}{\psi_\F(G')}=\frac{1}{\ell}\sum_{i=1}^\ell\psi_{F_i}(G)\psi_{F_i}(G') \ .$$ Applying a Hoeffding bound we get 
    $$\Pr\left(\left|\Tilde{k}_\F(G,G') - k_\scD(G,G')\right|>\varepsilon\right)
    \leq 2e^{-2\varepsilon^2\ell}\,.$$
Note that the previous bound holds for a fixed pair $G$ and $G'$. We can apply it to each pair in the training sample to get the following result.
\begin{theorem}\label{thm:lovasz_hoeffding_finite_sample}
    Let $\varepsilon,\delta\in(0,1)$, $\scD$ be a distribution on $\G_n$ with full support, and let $S\subseteq \G_n$ be a finite set of graphs. If we sample $\F =(F_1,\dots,F_\ell)\sim \scD^\ell$ i.i.d. with $$\ell=\scO\left(\frac{\log(\sfrac{|S|}{\delta})}{\varepsilon^2}\right)$$ we can guarantee that $$\max_{G,G'\in S}\left|\Tilde{k}_\F(G,G') - k_\scD(G,G')\right|< \varepsilon$$ with probability at least $1-\delta$.
\end{theorem}
\begin{proof}
    We have to show that $$\Pr\left(\max_{G,G'\in S}\left|\Tilde{k}_\F(G,G') - k_\scD(G,G')\right|>\varepsilon\right)<\delta\,.$$
    By a union bound it is sufficient if
    $$\sum_{G,G'\in S}\Pr\left(\left|\Tilde{k}_\F(G,G') - k_\scD(G,G')\right|>\varepsilon\right)<\delta$$
    and by applying Hoeffding bound to each term in the sum get
    $|S|^2 2e^{-2\varepsilon^2\ell}<\delta$. Solving for $\ell$ yields that $\ell=\scO\left(\frac{\log(\sfrac{|S|}{\delta})}{\varepsilon^2}\right)$ is sufficient.
\end{proof}
\begin{corollary}
    Let $\varepsilon,\delta\in(0,1)$, $\scD$ be a distribution on $\G_n$ with full support. If we sample $\F = (F_1,\dots,F_\ell)\sim \scD^\ell$ i.i.d. with $$\ell=\scO\left(\frac{n^2+\log(\sfrac{1}{\delta})}{\varepsilon^2}\right)$$ we can guarantee that $$\max_{G,G'\in \G_n}\left|\Tilde{k}_\F(G,G') - k_\scD(G,G')\right|< \varepsilon$$ with probability at least $1-\delta$.
\end{corollary}
\begin{proof}
Apply \Cref{thm:lovasz_hoeffding_finite_sample} with $S=\G_n$. We upper bound the number of graphs with up to $n$ vertices as $|\G_n|\leq 2^{(n^2)}$.
\end{proof}
Hence, we achieve a bound for all graphs in $\G_n$ while sampling only $\scO(n^2)$ patterns.

While we stated the previously achieved bounds for kernels, we can easily transform them to bounds on the induced distances of the graph embeddings using \Cref{eq:kernel_distance}.
\begin{corollary}\label{cor:distance_bound}
     Let $\varepsilon,\delta\in(0,1)$, $\scD$ be a distribution on $\G_n$ with full support. If we sample $\F = (F_1,\dots,F_\ell)\sim \scD^\ell$ i.i.d. with $$\ell=\scO\left(\frac{n^2+\log(\sfrac{1}{\delta})}{\varepsilon^2}\right)$$ we can guarantee that for all $G,G'\in\G_n$ simultaneously
     $$\left|\|\psi_\F(G)-\psi_\F(G')\|^2 - \|\sqrt{D}(\psi_{n}(G) - \psi_{n}(G'))\|^2\right|< \varepsilon$$
     with probability at least $1-\delta$.
\end{corollary}
Thus, our results apply not only to kernel methods, but also to learning methods that use the graph embedding directly, such as multilayer perceptrons.

\subsection{Graphs with Unbounded Size}\label{sec:minkernel}
In this section, we generalise the previous results to the set of all finite graphs $\G_\infty$.  Theorem~\ref{thm:Lovászvector} holds for $G, G' \in \G_\infty$ and the mapping $\phi_\infty$ that maps each $G \in \G_\infty$ to an infinite-dimensional vector. 
The resulting vector space, however, is not a Hilbert space with the usual inner product. 
To see this, consider any graph $G$ that has at least one edge. 
Then $\homo(P_n, G) \geq 2$ for every path $P_n$ of length $n \in \mathbb{N}$.
Thus, the inner product $\langle \phi_\infty(G), \phi_\infty(G) \rangle$ is not finite.

To define a kernel on $\G_\infty$ without fixing a maximum size of graphs, i.e., restricting to $\G_n$ for some $n \in \mathbb{N}$, we define the countable-dimensional vector
$\phi^\downarrow_\infty(G) = \left( \homo_{v(G)} (F, G)\right)_{F \in \G_\infty}$
  where 
\[ \homo_{v(G)}(F, G) = 
\begin{cases}
    \homo(F, G) & \text{if } v(F) \leq v(G)\,, \\
    0           & \text{if } v(F) >    v(G)\,. \\
\end{cases}
\]

That is, $\phi^\downarrow_\infty(G)$ is the projection of $\phi_{\infty}(G)$ to the subspace that gives us the homomorphism counts for all graphs of \emph{size at most of $G$}.
Note that this is a well-defined map of graphs to a subspace of the $\ell^2$ space, i.e., sequences $(x_i)_i$ over $\mathbb{R}$ with $\sum_i |x_i|^2 < \infty$.
Hence, the kernel given by the canonical inner product $k^\downarrow_\infty(G, G') = \langle \phi^\downarrow_\infty(G), \phi^\downarrow_\infty(G') \rangle_{\ell^2}$
is finite and positive semi-definite. Note that we can rewrite 
$k^\downarrow_\infty(G, G') = k_{\min}(G,G') = \langle \phi_{n'}(G), \phi_{n'}(G') \rangle$
where $n' = \min\{v(G), v(G')\}$. While the first hunch might be to count patterns up to $\max\{v(G),v(G')\}$, this is not necessary to guarantee completeness. \begin{restatable}{lemma}{MinKernel}
\label{lemma:min_kernel}
$k_{\min}$ is a complete kernel on $\G_\infty$.
\end{restatable}The proof can be found in Appendix~\ref{section:appendixproofs}.

Given a sample of graphs $S$, we note that for $n = \max_{G \in S} v(G)$ we only need to consider patterns up to size $n$.\footnote{It is sufficient to go up to the size of the second largest graph.}
As the number of graphs of a given size $n$ is superexponential, it is impractical to compute all such counts. Hence, we propose to resort to sampling.

\begin{restatable}{theorem}{minKernelRandom}
\label{lem:mincomplete}
Let $\scD$ be a distribution on $\G_\infty$ with full support and $G\in\G_\infty$. Then $\phi^\downarrow_F(G) = \hom_{v(G)}(F,G)e_F$ with $F\sim \scD$ and the corresponding kernel are expectation-complete.
\end{restatable}
The proof can be found in Appendix~\ref{section:appendixproofs}.

Note that $k_{\min}$ has the following interesting practical property. If we train a kernel-based classifier on a sample $S\subseteq\G_n$ and want to classify a graph with size larger than $n$ we do not have to recompute the embeddings $\phi^\downarrow_\infty(G)$ for $G\in S$ as the terms corresponding to patterns with size $>n$ in the kernel are zero anyway.

\section{Computing Embeddings in Expected Polynomial Time} \label{sec:computation}
An expectation-complete graph embedding should be efficiently computable to be practical, otherwise we could simply use deterministic complete embeddings. In this section, we describe our main result achieving polynomial runtime in expectation.
The best known algorithm \citep{DIAZ2002291} to exactly compute $\hom(F,G)$ takes time \begin{equation}\label{eq:hom_runtime}
    \scO(v(F)v(G)^{\tw(F)+1})
\end{equation} where $\tw(F)$ is the \emph{treewidth} of the pattern graph $F$. Thus, a straightforward sampling strategy to achieve polynomial runtime in expectation is to give decreasing probability mass to patterns with higher treewidth. Unfortunately, in the case of $\G_\infty$, this is not possible.
\begin{restatable}{proposition}{impossiblityLemma}
\label{lemma:impossiblity}
There exists no distribution $\scD$ with full support on $\G_\infty$  such that the expected runtime of Eq. (\ref{eq:hom_runtime}) becomes polynomial in $v(G)$ for all $G\in\G_\infty$.
\end{restatable}
The proof can be found in Appendix~\ref{section:appendixproofs}.

To resolve this issue we have to take the size of the largest graph in our sample into account. For a given sample $S\subseteq \G_n$ of graphs, where $n$ is the maximum number of vertices in $S$, we can construct simple distributions achieving polynomial time in expectation.
\begin{restatable}{theorem}{completeRandomPolynomial}\label{thm:completeRandomPolynomial}
There exists a distribution $\scD$ with full support on $\G_n$ such that computing the expectation-complete graph embedding $\phi^\downarrow_F(G)$ with $F\sim\scD$ takes polynomial time in $v(G)$ in expectation for all $G\in\G_n$.
\end{restatable}
\begin{proofsketch} We first draw a treewidth upper bound $k$ from an appropriate distribution. For example, to satisfy a runtime of $\scO(v(G)^{d+1})$ in expectation for some constant $d\in\naturals$, a Poisson distribution with $\lambda \leq \frac{1+d\log n}{n}$ is sufficient for any $G\in\G_n$. We have to ensure that each possible graph with treewidth up to $k$ gets a nonzero probability of being drawn. For that we first draw a $k$-tree---a maximal graph of treewidth $k$---and then take a random subgraph of it. See \Cref{section:appendixproofs} for the full proof.
\end{proofsketch}

We do not require that the patterns are sampled uniformly at random. It merely suffices that each pattern has a nonzero probability of being drawn.
We get a similar result for our random Lovász embedding.

\begin{restatable}{theorem}{completeRandomPolynomialMaxKern}\label{thm:completeRandomPolynomialMaxKern}
There exists a distribution $\scD$ with full support on $\G_n$ such that computing the expectation-complete graph embedding $\phi_F(G)$ with $F\sim\scD$ takes polynomial time in $v(G)$ in expectation for all $G\in\G_n$.
\end{restatable}
The proof can be found in \Cref{section:appendixproofs}.

Combining these results with \Cref{thm:lovasz_hoeffding_finite_sample}, we see that for any fixed $\delta$ and $\varepsilon$ we need in total an expected polynomial runtime to construct the embedding $\phi_\F$ with $\F=(F_1,\dots,F_\ell)$ with $F_i\sim\scD$ for all $i\in\{1,\dots,\ell\}$ and $\ell$ as in \Cref{thm:lovasz_hoeffding_finite_sample}.

\section{Practical Application}
\label{sec:practical}

So far, we have restricted our discussion to graphs without node attributes.
However, many real world datasets have attributes on their vertices and edges.We now discuss how to apply our embedding and kernel in such contexts.

It is conceptually possible to devise sampling schemes and corresponding distributions $\scD$ over graphs with discrete vertex and edge labels.
However, in practice this tends to result in unusable probabilities.
For any pattern $F$, a single edge with labeled endpoints which are not connected in $G$ results in $\homo(F,G) = 0$.
Hence, the resulting graph embeddings $\phi_\F$ become very sparse and practically uninformative.

We instead propose to consider labeled graphs as unlabeled for the purpose of homomorphism counting and suggest to include attribute information by applying a message passing graph neural network (GNN). 
Combining any GNN graph level representation with our embedding for a fixed set of sampled patterns $\F$ as shown in Figure~\ref{fig:architecture} is straightforward and allows to make any GNN architecture more expressive.
In particular the direct sum of $\phi_\F$ and the GNN representation is expectation-complete on attribute-free graphs; a property that the GNN representation alone does not posess.
Theorem~\ref{thm:universal} then implies that we can approximate any function on $\G_n$ using a suitable MLP with high probability.

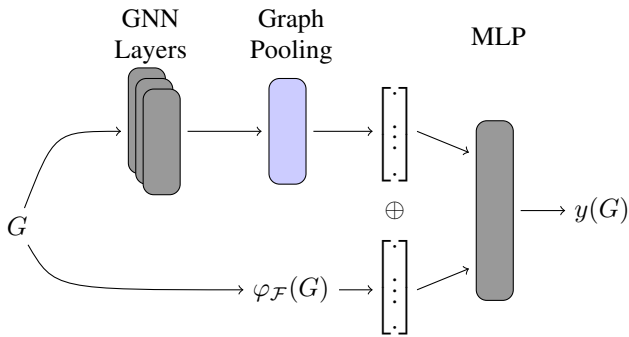
\begin{figure}

\begin{tikzpicture}[x=28pt, y=40pt]

\node (G) at (-1.5, -0.5) {$G$};

\coordinate (x) at (0,0);
\draw[black, fill=gray!80, rounded corners] ($ (x) + (0,0) $) rectangle ($ (x) + (0.5,1) $);

\node[align=center] at (0.3, 1.3) {GNN \\ Layers};

\draw[->] (G) .. controls (-1,0.4) .. (-0.1,0.4); 

\coordinate (x) at (0.1,-0.1);
\draw[black, fill=gray!80, rounded corners] ($ (x) + (0,0) $) rectangle ($ (x) + (0.5,1) $);

\coordinate (x) at (0.2,-0.2);
\draw[black, fill=gray!80, rounded corners] ($ (x) + (0,0) $) rectangle ($ (x) + (0.5,1) $);

\coordinate (x) at (1.9,-0.1);
\draw[black, fill=blue!20, rounded corners] ($ (x) + (0,0) $) rectangle ($ (x) + (0.5,1) $);

\node[align=center] at (2.2, 1.3) {Graph \\ Pooling};
\draw[->] (0.8,0.4) -- (1.8,0.4);
\draw[->] (2.5,0.4) -- (3.3,0.4);

\node at (3.6,0.25) {$ \begin{blockarray}{[c]}
\cdot \\
\vdots \\
\cdot \\
\end{blockarray} $};

\node (phi) at (2.2,-1.1) {$\phi_\F (G)$};
\draw[->] (G) .. controls (-1,-1.1) .. (phi);
\draw[->] (phi) -- (3.3, -1.1);

\node at (3.6,-1.2) {$ \begin{blockarray}{[c]}
\cdot \\
\vdots \\
\cdot \\
\end{blockarray} $};

\node at (3.6, -0.35) {$\oplus$};

\coordinate (x) at (4.7,-1.2);
\draw[black, fill=gray!80, rounded corners] ($ (x) + (0,0) $) rectangle ($ (x) + (0.5,1.7) $);
\node at (5,1.3) {MLP}; 

\draw[->] (3.9,0.4) -- (4.6,0.2);
\draw[->] (3.9,-1.1) -- (4.6,-0.9);

\node (y) at (6.4,-0.35) {$y(G)$};

\draw[->] (5.3,-0.35) -- (y);

\end{tikzpicture}

 \caption{Architecture of combining expectation-complete embeddings with MPNN representations for graph learning.}
\label{fig:architecture}
\end{figure}
 \section{Discussion and Related Work}
\label{sec:related}
\paragraph{$k$-WL test}
The $k$-dimensional Weisfeiler-Leman (WL) test\footnote{This refers to the \emph{folklore} $k$-WL test, also called $k$-FWL.} \citep{cai1992optimal} and the Lovász vector restricted to the set $\mathcal{T}_k$ of graph patterns with treewidth up to $k$ are equally expressive \citep{dvovrak2010recognizing, dell2018lov}, that is, they distinguish the same non-isomorphic graphs. \citet{puny2023equivariant} discuss this relationship in the context of invariant polynomials.
We now propose a random graph embedding with expected polynomial runtime that matches the expressiveness of $k$-WL in expectation. The same holds for MPNNs and $k$-GNNs, as their expressiveness is bounded by $k$-WL \citep{xu2018how, morris2019weisfeiler}. Let $\scD$ be a distribution with full support on $\mathcal{T}_k$ and  $\phi_F^{k\text{-WL}}(\cdot)$ be the resulting random graph embedding where $F\sim \scD$. \begin{restatable}{theorem}{completekWL}\label{thm:kwlinexpectation}
 The graph embedding $\phi_F^{k\text{-WL}}(\cdot)$ has the same expressiveness as the $k$-WL test in expectation for any $\scD$ that has full support on $\mathcal{T}_k$. Furthermore, there is a specific such distribution $\scD$ such that we can compute $\phi_F^{k\text{-WL}}(G)$ in expected polynomial time $\scO(v(G)^{k+1})$ for all $G\in\G_\infty$.
\end{restatable}
\Cref{lemma:impossiblity} does not apply to the embedding $\phi_F^{k\text{-WL}}(\cdot)$. In particular, the used distribution, which guarantees expected polynomial runtime, is independent of $n$ and can be used for all $\G_\infty$.

As before, we can state Hoeffding-based bounds to approximate how close we are to the full embedding $\phi_{\mathcal{T}_k}$.
\citet{morris2017glocalized} achieved similar bounds by sampling the $k$-tuple neighbourshoods of the $k$-WL test instead of the homomorphism counts.
\paragraph{Homomorphism-based graph embeddings.}
\citet{dell2018lov} proposed a complete graph kernel based on homomorphism counts related to our $k_{\min}$ kernel. Instead of implicitly restricting the embedding to only a finite number of patterns, as we do, they weigh the homomorphism counts such that the inner product defined on the whole Lovász vectors converges. However, \citet{dell2018lov} do not discuss how to compute their kernel and so, our approach can be seen as an efficient sampling-based alternative to their theoretical weighted kernel. 

Using graph homomorphism counts as a feature embedding for graph learning tasks was proposed by \citet{hoang2020graph} and \citet{kuehner2021graph}. \citet{hoang2020graph} discuss various aspects of homomorphism counts important for learning tasks, in particular, universality aspects and their power to capture certain properties of graphs, such as bipartiteness. Instead of relying on sampling patterns, which we use to guarantee expectation in completeness, they propose to use a small number of fixed pattern graphs. This limits the practical usage of their approach due to computational complexity reasons. In their experiments the authors only use tree (GHC-tree(6)) and cycle patterns (GHC-cycle(8)) up to size 6 and 8, respectively, whereas we allow patterns of arbitrary size and treewidth, guaranteeing polynomial runtime in expectation. Similarly to \citet{hoang2020graph}, we use the computed embeddings as features for a kernel SVM (with RBF kernel) and an MLP. For first results using an SVM, see our preliminary work at \citet{welke2022expectation} and \citet{thiessen2022expectation}. 

Instead of embedding the whole graph into a vector of homomorphism counts, \citet{barceloGraphNeuralNetworks2021} proposed to use rooted homomorphism counts as node features in conjunction with a graph neural network (GNN). They discuss the required patterns to be as or more expressive than the $k$-WL test. We achieve this in expectation when selecting an appropriate sampling distribution, as discussed above.

\paragraph{Cut distance}
The distance induced by the Lovász vector of all homomorphism counts is strongly related to the \emph{cut distance} \citep{borgs2006graph, lovasz2012book}. The cut distance is a well-studied and important distance on graphs that captures global structural but also sampling-based local information. It is well known that the distance given by homomorphism counts is close to the cut distance and hence has similar favourable properties. The cut distance, and hence homomorphism counts, capture the behaviour of all permutation-invariant functions on graphs. Using \Cref{cor:distance_bound} we see that this also holds for random embeddings, as they converge to the distance induced by the Lovász vector with high probability. For a discussion on the importance of the cut distance and homomorphism counts in the context of graph learning see \citet{dell2018lov}, \citet{klopp2019optimal}, \citet{grohe2020Word}, and \citet{hoang2020graph}.
\paragraph{Random graph and node embeddings}
\citet{wu2019scalable} adapted random Fourier features \citep{rahimi2007random} to graphs and proposed a sampling-based variant of the global alignment graph kernel. Similar sampling-based ideas were discussed before for the graphlet kernel \citep{shervashidze2009efficient, ghanem2021fast} and frequent-subtree kernels \citep{welke2015probabilistic}. The standard analysis of \citet{rahimi2007random} does not apply in our situation, as they require a shift-invariant kernel. Also the analysis by \citet{wu2019scalable} does not apply here, as they use finite-dimensional node embeddings as a starting point. None of the previously mentioned papers discusses random graph features in the context of expressiveness or completeness.
\citet{fang2021structure} and \citet{choromanski2023taming} considered random features for node embeddings and node classification tasks.
\paragraph{Random node initialisation} Instead of randomly embedding the whole graph, \citet{abboud2021surprising} and \citet{sato2021random} considered to initialise the vertices of the graphs with random labels. Through this they achieve universality in expectation. However, while for each realization of the random graph pattern $F$ our graph embedding $\phi_F$ is  universal in expectation and permutation-invariant, random node initialisation is only permutation-invariant in expectation.
\paragraph{Subgraph counts}
While subgraph counts are also a reasonable choice for expectation-complete graph embeddings, they have multiple drawbacks compared to homomorphism counts. Most importantly,  from a computational perspective, computing subgraph counts even for graphs such as trees or paths is NP-hard \citep{alon1995color, marx2014Everything}, while we can compute homomorphism counts efficiently  for pattern graphs with small treewidth \citep{DIAZ2002291}. In particular, all known exact algorithms for (induced) subgraph isomorphism counting have a worst-case runtime of $\scO(v(G)^{v(F)})$, even for patterns with small treewidth. This one of the main reasons why the graphlet kernel \citep{shervashidze2009efficient} and similar fixed pattern based approaches \citep{bouritsas2022improving} only count subgraphs up to size around 5 or are only sufficient. Alternative approaches exist, such as the cyclic pattern kernel \citep{horvath2004cyclic} and the neighbourhood-based kernel of \citet{costa2010fast}, that are efficiently computable in special cases, for example on most molecular graphs. %
 \section{Empirical Evaluation}
\label{sec:experiments}

We analyze the performance of our expectation-complete embedding that can be computed in expected polynomial time.
The details of the pattern sampling process are described in \Cref{sec:samplingdetails}.
We evaluate our proposed embeddings in two contexts.
We investigate how graph embeddings from message passing graph neural network (GNN) perform when augmented with our embeddings. 
To complement these results, we investigate the empirical expressive power of our embeddings on synthetic benchmark datasets. 
The code to sample patterns and to compute representations\footnote{Representations: \href{https://github.com/pwelke/homcount}{github.com/pwelke/homcount}}, as well as for the GNN experiments\footnote{GNN evaluation: \href{https://github.com/ocatias/HomCountGNNs}{github.com/ocatias/HomCountGNNs}} is available.

\begin{table}
    \centering
    \caption{Performance of different GNNs on 9 OGB benchmarks and \texttt{ZINC}. Baseline of a GNN with homorphism counts is the same GNN without homomorphism counts. Results for GNNs with homorphism counts are averaged over 9 different random samples of pattern graphs.}
    \label{tab:ogb_performance}
    \vskip 0.15in
    \begin{tabular}{lll}
\toprule
 & Top 1 / \hspace{0.25cm}2\hspace{0.25cm} / \phantom{0}3\phantom{0} & Beats baseline\\
 \midrule
GIN & \phantom{00}0\% / \phantom{0}0\% / \phantom{0}0\% & - \\
\homGNN{GIN} & \phantom{00}0\% / 10\% / 10\% & 100\%\\
GCN & \phantom{00}0\% / \phantom{0}0\% / \phantom{0}0\% & - \\
\homGNN{GCN} & \phantom{0}10\% / 10\% / 20\% & 90\% \\
GIN+F & \phantom{00}0\% / 10\% / 50\% & - \\
\homGNN{GIN} +F & \phantom{0}20\% / 40\% / 70\% & 90\% \\
GCN+F & \phantom{00}0\% / 50\% / 60\% & - \\
\homGNN{GCN}+F & \phantom{0}70\% / 80\% / 90\% & 90\%\\
\bottomrule
\end{tabular}
\vskip -0.1in
\end{table}

\subsection{Improving GNNs with Graph-Level Homomorphism Counts}
For graph-level prediction tasks, GNNs compute a graph embedding which is used by an MLP to make the final prediction.
We propose to extend the learned graph embedding by concatenating it with a vector of homomorphism counts for a set of up to 50 sampled patterns $\F$ (cf. \Cref{sec:practical}).
As this approach is independent of the GNN it can boost the expressiveness of any GNN. Furthermore, it is possible to extend already trained GNNs by these patterns by simply changing the width of the MLP and fine tuning.
We denote GNNs boosted by homomorphism counts by ``GNN+hom''.
We compare two settings: with (``GNN+F'') and without (``GNN'') node and edge features.
We determine whether our approach reliably boosts the prediction accuracy of GNNs.

\paragraph{Models.}  
We use GIN \citep{xu2018how} and GCN \citep{GCN} as baseline GNNs.
We compare the baselines against \homGNN{GIN} and \homGNN{GCN}. When using homomorphism counts, we first train the model without these counts and then finetune the entire model with the full homomorphism vector. We normalize the vector of homomorphism counts such that each entry has 0 mean and a standard deviation of 1 over the training set. We base our hyperparameters on \citet{ogb} and tune only the dropout rate (for all hyperparameters see Table~\ref{tab:hyper_params} in Appendix~\ref{sec:app_exp}). For models without homomorphism counts, we train and evaluate a model 10 times for the best hyperparameters. For models with homomorphism counts, we first find the best hyperparameters for one sample of homomorphism counts. Then, we train and evaluate the model with these hyperparameters for 8 different samples of pattern graphs and thus different homomorphism counts. For each model, we report the test result in the corresponding epoch with the best validation metric (see Appendix~\ref{sec:app_exp}). We report the average and standard deviation of all test results for a given type of model.

\paragraph{Setup.} We evaluate on the commonly used molecule datasets \texttt{ZINC}, \texttt{ogbg-molhiv} and \texttt{ogbg-moltox21} \citep{ogb}. Furthermore, we also train on 7 additional small molecule datasets from \citet{ogb} (see Appendix~\ref{sec:app_exp}).
For \texttt{ZINC} we use the same setup as \citet{Cellular-WL}: we use a batch size of 128 and an initial learning rate of $10^{-3}$ which we reduce by half every 20 epochs without an improvement of the validation performance. We stop training after either 500 epochs or after the learning rate is smaller than $10^{-5}$. To finetune on \texttt{ZINC}, we restart the training procedure with an initial learning rate of $5\cdot 10^{-4}$. For datasets based on OGB, we train for 100 epochs with a batch size of 32 and a fixed learning rate of $10^{-3}$ which corresponds to the initial learning rate on \texttt{ZINC}. To finetune, we train for 100 additional epochs with a learning rate of $5\cdot10^{-4}$.  We perform an ablation study in Appendix~\ref{sec:app_exp}.

\paragraph{Results.} We summarize the results of the experiments in Table~\ref{tab:ogb_performance}. The center column shows how often the best parameter setting for a variant (e.g. \homGNN{GIN}+F) was among the top 1, top 2, or top 3 scoring models among the ten datasets. Recall, that this references the predictive performance on the test in the epoch with the best performance on the validation set. We can immediately see that including homomorphism information is helpful for predictive performance as the best performing model for \emph{every} dataset uses homomorphism counts.
For each model, the rightmost column reports if a GNN variant with homomorphism counts beats its respective baseline GNN without added homomorphism counts. We can see that models with homomorphism counts outperform the baseline in at least 90\% of the datasets. This demonstrates that besides theoretical guarantees, homomorphism counts also reliably improve the practical prediction performance of GNNs. Detailed results for all datasets and an ablation study can be found in Appendix~\ref{sec:app_exp}.

\begin{table}[t]
    \caption{Accuracy on synthetic data}
    \label{tab:synthetic}
    \vskip 0.15in
    \centering
    \begin{tabular}{lll}
    \toprule
    Method & CSL & PAULUS25 \\
    \midrule
    GIN  & 10.00 $\pm$ 0.00 & 7.14 $\pm$ 0.00 \\
    GNTK & 10.00 $\pm$ 0.00 & 7.14 $\pm$ 0.00 \\ 
    GHC-Tree  & 10.00 $\pm$ 0.00 & 7.14 $\pm$ 0.00 \\
    GHC-Cycle & 100.0 $\pm$ 0.00 & 7.14 $\pm$ 0.00 \\
    WL  & 10.00 $\pm$ 0.00 & 7.14 $\pm$ 0.00 \\
    \midrule
Ours &  37.67 $\pm$ 9.11 &  100.0 $\pm$ 0.00 \\
    \bottomrule
    \end{tabular}
\vskip -0.1in
\end{table}

\subsection{Expressiveness on Synthetic Datasets}
We complement these results on real world graph datasets with an empirical analysis of our approach on synthetic benchmark datasets used to evaluate the expressiveness of graph learning approaches. 
On these benchmarks the labels encode isomorphism classes. Both datasets are balanced and have 10 (CSL) and 14 (PAULUS25) classes, respectively. 
We sample a fixed number $\ell=50$ of patterns and compute the sampled min kernel (resp. the corresponding embedding) as described in Section~\ref{sec:minkernel}.
Table~\ref{tab:synthetic} shows averaged accuracies of an SVM classifier trained on our feature sets on the datasets CSL \citep{murphy2019relational} and PAULUS25 \citep{hoang2020graph}\footnote{Originally from \href{https://www.distanceregular.org/graphs/paulus25.html}{www.distanceregular.org/graphs/paulus25.html}}.
We follow the experimental design of \citet{hoang2020graph} and compare to their published results.
We also included GNTK \citep{du2019graph}, GIN \citep{xu2018how}, and the WL-kernel \citep{shervashidze2011weisfeiler}. 
Even with as little as 50 features, it is interesting to note that a SVM with RBF kernel and our features performs perfectly on the PAULUS25 dataset, i.e., it is able to decide isomorphism for the strongly regular graphs in this dataset. 
On the CSL dataset the min kernel performs better than all competitors except GHC-cycle, which was specifically designed for this dataset.
The performance of the min kernel on this dataset increases monotonically for larger number of patterns, for instance to 48.8\% for 200 patterns, see Appendix~\ref{appendix:CSL}.
 \section{Conclusion}
\label{sec:conclusion}
In this work, we introduced the notion of expectation-complete graph embeddings---random embeddings, which in expectation can distinguish any pair of non-isomorphic graphs. 
We studied their general properties and have shown that repeated sampling will eventually allow us to distinguish any fixed pair of non-isomorphic graphs, which results in a universal representation for graphs of bounded size.
We proposed to sample the Lovász vector of homomorphism counts as one possibility to achieve expectation-completeness and have shown favourable properties, such as bounds on the convergence of the random embedding to the full Lovász vector. 
Using a specific distribution which gives exponentially decreasing probability to patterns with large treewidth, we showed that computing our embedding takes polynomial time in expectation. 
We discussed that homomorphism counts of patterns with treewidth up to $k$ can be seen as a sampling-based variant of the $k$-WL test with the same expressiveness in expectation and that homomorphism counts are strongly related to the cut-distance. 
Our empirical results have shown that homomorphism counts of sampled patterns (a) tend to increase the performance of MPNNs on a set of benchmark datasets and (b) allow to learn classifiers that distinguish non-isomorphic graphs where MPNNs and other baselines fail.

As future work, we will investigate approximate counts to make our implementation more efficient \citep{beaujean2021graph}. 
It is unclear how this affects expressiveness, as we loose permutation-invariance. 
Similar to \citet{abboud2021surprising} we would still retain permutation-invariance in expectation. 
Going beyond expressiveness results, our goal is to further study graph similarities suitable for graph learning, such as the cut distance as proposed by \citet{grohe2020Word}. 
Finally, instead of sampling patterns from a fixed distribution, a more promising variant is to adapt the sampling process in a sample-dependent manner. 
One could, for example, draw new patterns until each graph in the sample has a unique embedding (up to isomorphism) or at least until we are at least as expressive as 1-WL on the given sample. 
Alternatively, we could pre-compute frequent or interesting patterns as proposed by \citet{schulz2018mining} and use them to adapt the distribution. 
Such approaches would use the power of randomisation to select an appropriate graph embedding in a data-driven manner, instead of relying on a finite set of fixed and pre-determined patterns like previous work \citep{barceloGraphNeuralNetworks2021, bouritsas2022improving}. \section*{Acknowledgements}
Part of this work has been funded by the Vienna Science and Technology Fund (WWTF) project ICT22-059 as well as by the Federal Ministry of Education and Research of Germany and the state of North Rhine-Westphalia as part of the Lamarr Institute for Machine Learning and Artificial Intelligence, LAMARR22B. 
\bibliography{references}
\bibliographystyle{icml2022}

\newpage
\appendix
\onecolumn
\section{Proofs}\label{section:appendixproofs}

\MinKernel*
\begin{proof}
Let $G, G' \in \G_\infty$.
We have to show that 
\[ \phi_F^{\downarrow}(G) = \phi_F^{\downarrow}(G') \Leftrightarrow G \isomorphic G'  \ , \]
We start by the $\Rightarrow$ direction.
There are two cases:
\begin{enumerate}
    \item $v(G)=v(G')$: 
    
    Then, by Theorem~\ref{thm:Lovászvector} we have $\phi_n(G) = \phi_n(G')$ iff $G \isomorphic G'$ for $n = \min\{v(G), v(G')\} = v(G) = v(G')$.

    \item $v(G) \neq v(G')$: 
    
    Let w.l.o.g. $0 < v(G) < v(G')$.
    Let $P$ be the graph with exactly one vertex.
    Then $\homo(P, G) < \homo(P, G')$, i.e., we can distinguish graphs on different numbers of vertices using homomorphism counts. 
    As $\min\{v(G), v(G')\} \geq 1$, we have $P \in \G_{v(G)}$ and hence $\phi_{v(G)}(G) \neq \phi_{v(G)}(G')$.
\end{enumerate}
The $\Leftarrow$ direction follows directly from the fact that homomorphism counts are invariant under isomorphism. 
\end{proof}

\minKernelRandom*
\begin{proof}
We can apply the same arguments as before from Theorem \ref{thm:Lovasz_complete_exp} to show that the expected embeddings of two graphs $G,G'$ with size $n'=\min\{v(G),v(G')\}$ are equal iff their Lovász vector restricted to size $n'$ are equal. By Lemma \ref{lemma:min_kernel} we know that the latter only can happen if the two graphs are isomorphic.
\end{proof}

\impossiblityLemma*
\begin{proof}
Let $\scD$ be such a distribution and let $\scD'$ be the marginal distribution on the treewidths of the graphs given by $p_k = \Pr_{F\sim\scD}(\tw(F)=k)>0$. Let $G$ be a given input graph in the sample. \citet{DIAZ2002291} have shown that computing $\homo(F,G)$ takes time $\scO\left(v(F)v(G)^{\tw(F)+1}\right)$. Assume for the purpose of contradiction that we can guarantee an expected polynomial runtime (ignoring the $v(F)$ term and constant factors for simplicity): 
$$\E_{F\sim \scD}[v(G)^{\tw(F)+1}]=\sum_{k=1}^\infty p_k v(G)^{k+1}\leq Cv(G)^c$$ for some constants $C,c\in\naturals$. Then for all $k\geq c$, it must hold that $p_k v(G)^{k+1} \leq Cv(G)^c$, as all summands are positive. However, for large enough $v(G)$ the left hand side is larger than the right hand side. Contradiction.
\end{proof}

\completeRandomPolynomial*
\begin{proof}
We draw a treewidth upper bound $k$ from a Poisson distribution with parameter $\lambda$ to be determined later. Select a distribution $\scD_{n,k}$ which has full support on all graphs with treewidth up to $k$ and size up to $n$, for example, the one described in Appendix \ref{sec:samplingdetails}. 
Let $G\in\G_n$. Note that $\scD_{n,k}$ is a fixed distribution for $\G_n$ and independent of $G$. Also, $\phi^\downarrow_F(G)=0$ for all patterns $F$ with $v(F)>v(G)$ by definition of $\phi^\downarrow$. Hence, in this case we do not have to run the homomorphism counting computation. Thus, we can restrict the support of $\scD_{n,k}$ to $\scD_{v(G),k}$. Overall, the runtime is determined by the homomorphism counting. Using the algorithm of \citep{DIAZ2002291} this gives, for some constant $C\in\naturals$, an expected runtime of 
\begin{eqnarray*}
\E_{k\sim\Poi(\lambda),F\sim \scD_{v(G),k}}\left[Cv(F)v(G)^{\tw(F)+1}\right] & \leq & \E_{k\sim\Poi(\lambda)}\left[Cv(G)^{k+2}\right] \\
& = & \sum_{k=0}^\infty \frac{\lambda^k e^{-\lambda}}{k!}Cv(G)^{k+2}=\frac{Cv(G)^2}{e^\lambda}e^{\lambda v(G)}. \\
\end{eqnarray*}
We need to bound the right hand side by some polynomial $Dv(G)^d$ for some constants $D,d\in\naturals$. By rearranging terms we see that $$\lambda \leq \frac{\ln\frac{D}{C}+(d-2)\ln v(G)}{v(G)-1} $$is sufficient. To satisfy this inequality for all graphs in $\G_n$ simultaneuosly (meaning for all possible graph sizes up to $n$) we choose $$\lambda \leq \frac{\ln\frac{D}{C}+(d-2)\ln n}{n-1}\,,$$ which is valid as the right hand side is monotonically decreasing in $v(G)$.

\end{proof}

\completeRandomPolynomialMaxKern*
\begin{proof}
    The proof proceeds almost exactly as the one for \Cref{thm:completeRandomPolynomial}. However, to guarantee a runtime polynomial in $\scO(v(G))$ instead of merely $\scO(n)$ we have to proceed slightly more careful. 

    As before in \Cref{thm:completeRandomPolynomial} we draw $k$, the treewidth upper bound, from a Poisson distribution with parameter $\lambda$ and $F$ from a distribution $\scD_{n,k}$. Here, we additionally, require that $\E_F[v(F)]$ is bounded by a constant (while still having full support on the graphs of treewidth up to $k$ and size up to $n$), which is easily satisfied by for example a Geometric distribution (with its parameter set to some constant) restricted to $\{1,\dots,n\}$. We see that
    \begin{align*}
        \E_{k\sim\Poi(\lambda), F\sim \scD_{n,k}}[v(F)v(G)^{\tw(F)+1}] &= \sum_{i=1}^\infty \sum_{F'\in\G_n} \Pr(k=i,F=F') v(F)v(G)^{\tw(F)+1} \\
        &\leq\sum_{i=1}^\infty \sum_{F'\in\G_n} \Pr(k=i)\Pr(F=F'|k=i) v(F)v(G)^{k+1}\\
        &=\sum_{i=1}^\infty  \Pr(k=i)v(G)^{k+1}\sum_{F'\in\G_n} \Pr(F=F'|k=i)v(F)\\
        &=\sum_{i=1}^\infty  \Pr(k=i)v(G)^{k+1}\E_{F\sim \scD_{n,i}}[v(F)]\,.
        \end{align*}
        As the expectation of $v(F)$ is by assumption bounded by a constant, it remains to upper bound $\sum_{i=1}^\infty  \Pr(k=i)v(G)^{k+1}$ by a polynomial in $v(G)$. This can be achieved by choosing $\lambda$ as before in \Cref{thm:completeRandomPolynomial}.
\end{proof}

\section{Matching the Expressiveness of $k$-WL in Expectation}\label{appendix:kWL}
We devise a graph embedding matching the expressiveness of the $k$-WL test in expecation.
\completekWL*
\begin{proof}
Let $\mathcal{T}_k$ be the set of graphs with treewidth up to $k$ and $\scD$ be a distribution with full support on $\mathcal{T}_k$.
Then by the same arguments as before in Theorem \ref{thm:Lovasz_complete_exp}, the expected embeddings of two graphs $G$ and $G'$ are equal iff their Lovász vectors restricted to patterns in $\mathcal{T}_k$ are equal. By \citet{dvovrak2010recognizing} and \citet{dell2018lov} the latter happens iff $k$-WL returns the same color histogram for both graphs. This proves the first claim.

For the second claim note that the worst-case runtime for any pattern $F\in\mathcal{T}_k$ is $\scO\left(v(F)v(G)^{k+1}\right)$ by \citet{DIAZ2002291}. However, the equivalence between homomorphism counts on $\mathcal{T}_k$ and $k$-WL requires to inspect also patterns $F$ of all sizes, in particular, also larger than the size $n$ of the input graph. To remedy this, we can draw the pattern size $m=v(F)$ from some distribution with bounded expectation and full support on $\naturals$. For example, a geometrically distributed $m\sim\Geom(p)$ with any constant parameter $p\in(0,1)$ and expectation $\E[m] = \frac{1}{1-p}$ is sufficient. By linearity of expectation then
\begin{align*}
    \E_{F\sim\scD}\left[v(F)v(G)^{\tw(F)+1}\right]&\leq\E_{F\sim\scD}\left[v(F)v(G)^{k+1}\right]\\
    &= \E_{F\sim\scD}\left[v(F)\right]v(G)^{k+1}\\
    &=\scO\left(v(G)^{k+1}\right)\,.
\end{align*}

\end{proof}

\begin{algorithm}[t]
\caption{Sampling algorithm for a pattern set}
\label{alg:sampling}
\renewcommand{\algorithmiccomment}[1]{\hfill // #1}
\newcommand{\RETURN}[0]{\STATE\textbf{return} }
\begin{description}
\item[input:] the maximum graph size $n$, a number $\ell$ of requested patterns
\item[output:] a list $P$ of $\ell$ patterns
\end{description}
\begin{algorithmic}[1]
    \STATE Initialize $P = \{\text{singleton, edge, wedge, triangle}\}$ 
    \FOR{$i=5$ to $\ell$}
        \STATE Draw pattern size $N \sim \Geom(1 - \sqrt[n-3]{0.01}) + 3$
        \STATE Draw pattern treewidth $k \sim \Poi(\frac{1+\log(n)}{n}) + \Uni(3)$
        \STATE $k = \min(k, N-1)$ \COMMENT{maximum treewidth of graph on $N$ vertices is $N-1$}
        \STATE Sample a maximal graph $F$ with treewidth $k$ and $N$ vertices
        \FOR{$e \in E(F)$}
        \STATE Remove $e$ from $F$ with probability $p=0.1$
        \ENDFOR
        \STATE Add $F$ to $P$
    \ENDFOR
    \RETURN $P$
\end{algorithmic}
\end{algorithm}

\section{Sampling Details}\label{sec:samplingdetails}
To obtain a practical sampling algorithm for unlabeled graphs that draws from a distribution with full support on $\G_n$, we proceed as follows:
We first draw a pattern size $N \leq n$ from some distribution $\scD_1$ and then draw a treewidth upper bound $k < N$ from some distribution $\scD_2$. 
Then we want to sample any graph with treewidth at most $k$ and $N$ vertices with a nonzero probability. 
While guaranteed uniform sampling would be preferable, we resort to a simple sampling scheme that is easy to implement. 
A natural strategy is to first sample a \emph{$k$-tree}, which is a maximal graph with treewidth $k$, and then take a random subgraph of it.
Uniform sampling of $k$-trees is described by \citet{nie2015learning} and \citet{caminiti2010bijective}. 
Alternatively, the strategy of \citet{yoo2020sampling} is also possible. 
Note that we only have to guarantee that each pattern has a nonzero probability of being sampled; it does not have to be uniform. 
We achieve a nonzero probability for each pattern of at most a given treewidth $k$ by first constructing a random $k$-tree $P$ through its tree decomposition, by uniformly drawing a tree $T$ on $N-k$ vertices and choosing a root. 
We then create $P$ as the (unique up to isomorphism) $k$-tree that has $T$ as tree decomposition. 
We then randomly remove edges from $P$ i.i.d. with fixed probability $p_{\text{removal}} > 0$. This ensures that each subgraph of $P$ will be created with nonzero probability.

We choose $\scD_1$ as a geometric distribution with parameter $p=1 - \sqrt[n]{0.01}$ to ensure that $N \leq n$ with probability $0.99$. While we require in \Cref{thm:completeRandomPolynomialMaxKern} that the expectation of $N$ is bounded by a constant to satisfy an expected runtime that is polynomial in each graph size $v(G)$, this sampling only guarantees a runtime polynomial in $n$ (the upper bound on all graphs in the training sample), as $v(F)$ goes linearly in the runtime and could exceed $v(G)$. For the whole training sample this still has an expected polynomial runtime.
The min kernel, however, can be computed in polynomial runtime in $v(G)$ in expectation for this sampling scheme.
To achieve polynomial runtime in expectation, we choose $\scD_2$ as a Poisson distribution with parameter $\lambda=\frac{1+\log(n)}{n}$, where we add a number from the set $\Set{1,2,3}$ (drawn uniformly at random) to the outcome. 
This ensures polynomial runtime in expectation, while it increases the probability of drawing nontrivial treewidths.
Finally, we set $p_{\text{removal}} =  0.1$.
This sampling scheme assigns high probability to small graphs with low treewidth.
When drawing multiple patterns (e.g. $l=50$ as in our experiments), we observe that small patterns of size up to three are typically drawn multiple times.
We don't filter out isomorphic patterns to maintain expected polynomial time.
Instead, to practically improve the pattern distribution, we include the four nonisomorphic patterns up to size three (singleton, edge, wedge, and triangle) deterministically and draw the remaining pattern sizes from $\tilde{\scD}_1 = \Geom(1 - \sqrt[n-3]{0.01}) + 3$. 
Algorithm~\ref{alg:sampling} summarizes the sampling process for a fixed number of $\ell$ patterns.

\section{Experimental Details}
\label{sec:app_exp}

Our source code for the pattern sampling and homomorphism counting is available on github\footnote{Pattern sampling and representations: \href{https://github.com/pwelke/homcount}{github.com/pwelke/homcount}}.
The repository additionally contains the evaluation code for the synthetic datasets, with an SVM using an RBF kernel and started as a fork of the code of \citet{hoang2020graph}\footnote{Code of \citeauthor{hoang2020graph}: \href{https://{github.com/gear/graph-homomorphism-network}}{github.com/gear/graph-homomorphism-network}}.
We rely on the C++ code of \citet{curticapean2017homomorphisms}\footnote{Homomorphism counting: \href{https://github.com/ChristianLebeda/HomSub}{github.com/ChristianLebeda/HomSub}} to efficiently compute homomorphism counts. 
While the code computes a tree decomposition itself we decided to simply provide it with our tree decomposition of the $k$-tree which we compute as part of the sampling process, to make the computation more efficient\footnote{Adapted version of homomorphism counting: \href{https://github.com/pwelke/HomSub}{github.com/pwelke/HomSub}}. 
The datasets in the correct format as well as the sampled graph representations used for the evaluation in this paper can be downloaded\footnote{Links to download datasets and graph representations: \href{https://github.com/pwelke/homcount}{github.com/pwelke/homcount}} to reproduce the experiments in this paper.

\begin{table}
\begin{center}
    \begin{minipage}[b][5.2cm][t]{5cm}
    \vfill
    \begin{center}
    \caption{
    Average accuracy of an SVM using our min kernel on CSL for varying number of patterns $l$ over ten independent samples of patterns.
    }
\vskip 0.15in
    {
    \label{tab:CSLscaling}
    \begin{tabular}{ll}
    \toprule
    $\ell$ & Accuracy \\
    \midrule
	20 & 28.3\% $\pm$ 6.7\% \\
	50 & 35.0\% $\pm$ 8.6\% \\
	100 & 38.7\% $\pm$ 8.4\% \\
	150 & 44.6\% $\pm$ 6.2\% \\
	200 & 48.8\% $\pm$ 3.4\% \\
    \bottomrule
    \end{tabular}
    }
    \end{center}
    \end{minipage}\hspace{2cm}
\begin{minipage}[b][5.2cm][t]{8cm}
\vfill
\begin{center}
\caption{
Hyperparameter grid used for all models.
}
\vskip 0.15in
{
\label{tab:hyper_params}
\begin{sc}
\begin{tabular}{ll}
\toprule
\tablecell{Parameter $\downarrow$} & \tablecell{Values $\downarrow$}  \\
\midrule
Embedding Dimension\ & 300\\ 
Number of GNN Layers & 5 \\ 
Number of MLP Layers & 2 \\ 
Dropout Rate & 0, 0.5\\ 
Pooling Operation & mean \\ 
\bottomrule
\end{tabular}
\end{sc}
}
\end{center}
\end{minipage}
\end{center}
\end{table}

The source code for the GNN evaluation is available in a separate github repository\footnote{GNN evaluation: \href{http://github.com/ocatias/HomCountGNNs}{github.com/ocatias/HomCountGNNs}}. We implement our models in PyTorch and PyTorch Geometric and train on a single NVIDIA GeForce RTX 3080 GPU.
We evaluate on the molecular graph level tasks from the Open graph Benchmark \citep{ogb} and on the \texttt{ZINC} subset dataset provided by \citet{DBLP:journals/corr/Gomez-Bombarelli16}.
We use the provided train/validate/test splits.
\Cref{tab:hyper_params} shows the values of the hyperparameters used for each of the ten datasets. The hyperparameters are based on \citet{ogb} with the difference that we are using two layers in the final MLP instead of one as we have found this to yield significantly better predictions in preliminary experiments. \Cref{tab:full_results_OGB} shows the relevant metrics and results for all datasets.

\paragraph{Ablation Study.}  We perform an ablation study to investigate whether improvements from homomorphism counts are due to the homomorphism counts or because we finetune the model for more epochs. For this, we train GIN with misaligned homomorphism counts i.e., counts that were computed for a different graph. We denote this as GIN+MIS. We compare the performance of baselines against models with (misaligned) features on all datasets with features except \texttt{ZINC} and \texttt{ogbg-molhiv} due to their large size. Table~\ref{tab:GNN_ablations} shows the results of the ablations. We see (1) that GIN with homomorphism counts always outperforms GIN with misaligned features. Interestingly, (2) GIN with misaligned features outperforms the baseline GNN in 7 out of 8 datasets. From (2) it follows that some of the improvements of using homomorphism counts is dude to the additional finetuning. However, from (1) it follows that much of improvement of using homomorphism counts is due to the homomorphism counts and not because of the finetuning.

\label{appendix:CSL}
\paragraph{Stability and Increasing Sample Size on CSL}
We investigate the stability of predictive performance over independent pattern sets as well as the performance of the min kernel for larger samples.
We repeat the experiment on CSL shown in Table~\ref{tab:synthetic} ten times each for varying number of sampled patterns $\ell \in \Set{20,50,100,150,200}$.
Table~\ref{tab:CSLscaling} shows the results.
The target of the prediction task on CSL is the isomorphism class of the input graph.
Our theory implies that the expressiveness of the representations increases with number of samples. 
We indeed see that the performance on CSL increases with number of sampled patterns.
Furthermore, the variance over different sets drops with increasing number of sampled patterns. 
Thus, at least on CSL, the performance of an SVM using the min kernel is not very sensitive to the pattern set and increases with the number of samples.
However, we note that we could design a distribution over cycle graph patterns alone which would allow to perfectly distinguish the isomorphism classes on CSL.

\begin{table}[t]
\centering
\caption{Results of different GNNs on molecular datasets. \textbf{Bold} results are GNNs with homomorphism counts that are better than the same GNN without homomorphism counts. Results with homomorphism counts are averaged over 9 different samples of pattern graphs.}
\label{tab:full_results_OGB}
\vskip 0.15in
\begin{tabular}{llllll}\toprule

\tablecell{ DATA SET $\rightarrow$ \\ $\downarrow$ MODEL} & \tablecell{MOLBACE \\ roc-auc $\uparrow$} & \tablecell{MOLCLINTOX \\ roc-auc $\uparrow$} & \tablecell{MOLBBBP\\ roc-auc$\uparrow$}  & \tablecell{MOLSIDER\\ roc-auc $\uparrow$} & \tablecell{MOLTOXCAST \\ roc-auc $\uparrow$} \\
\midrule
GIN & $82.2 \pm 2.0$ & $61.2 \pm 4.5$ & $60.9 \pm 2.4$ & $57.5 \pm 1.4$ & $57.1 \pm 0.8$ \\
\homGNN{GIN} & $\mathbf{82.7 \pm 1.8}$ & $\mathbf{61.5 \pm 4.1}$ & $\mathbf{63.0 \pm 1.1}$ & $\mathbf{58.4 \pm 1.2}$ & $\mathbf{58.1 \pm 0.5}$ \\
GCN & $81.4 \pm 2.4$ & $68.4 \pm 3.6$ & $59.2 \pm 1.0$ & $58.2 \pm 1.3$ & $58.6 \pm 0.6$ \\
\homGNN{GCN} & $\mathbf{84.6 \pm 1.3}$ & $63.4 \pm 4.7$ & $\mathbf{61.2 \pm 0.7}$ & $\mathbf{59.2 \pm 1.2}$ & $\mathbf{59.4 \pm 0.4}$ \\
GIN+F & $75.5 \pm 3.0$ & $84.8 \pm 3.7$ & $67.2 \pm 1.5$ & $57.7 \pm 1.8$ & $61.8 \pm 0.8$ \\
\homGNN{GIN}+F & $\mathbf{76.4 \pm 2.6}$ & $\mathbf{86.9 \pm 3.5}$ & $\mathbf{68.8 \pm 1.3}$ & $\mathbf{58.4 \pm 1.5}$ & $\mathbf{63.2 \pm 0.8}$ \\
GCN+F & $82.2 \pm 1.4$ & $88.2 \pm 3.0$ & $66.4 \pm 2.6$ & $59.3 \pm 1.6$ & $65.7 \pm 0.4$ \\
\homGNN{GCN}+F & $81.3 \pm 1.6$ & $\mathbf{90.4 \pm 2.0}$ & $\mathbf{70.8 \pm 1.2}$ & $\mathbf{60.0 \pm 1.9}$ & $\mathbf{65.8 \pm 0.8}$ \\ \midrule
 &  \tablecell{MOLLIPO \\ rmse $\downarrow$} & \tablecell{MOLTOX21\\ roc-auc $\uparrow$} & \tablecell{MOLESOL\\ rsmse $\downarrow$}  &\tablecell{MOLHIV\\ roc-auc $\uparrow$} &   \tablecell{ZINC\\ mae $\downarrow$}  \\
 \midrule
GIN & $1.062 \pm 0.025$ & $65.4 \pm 1.9$ & $1.852 \pm 0.044$ & $69.1 \pm 2.2$ & $1.262 \pm 0.017$ \\
\homGNN{GIN} & $\mathbf{1.006 \pm 0.017}$ & $\mathbf{67.5 \pm 1.1}$ & $\mathbf{1.746 \pm 0.096}$ & $\mathbf{71.0 \pm 1.9}$ & $\mathbf{1.231 \pm 0.014}$ \\
GCN & $1.056 \pm 0.035$ & $66.7 \pm 0.7$ & $1.855 \pm 0.073$ & $69.1 \pm 2.2$ & $1.281 \pm 0.013$ \\
\homGNN{GCN} & $\mathbf{0.986 \pm 0.015}$ & $\mathbf{66.8 \pm 1.1}$ & $\mathbf{1.735 \pm 0.066}$ & $\mathbf{72.2 \pm 1.4}$ & $\mathbf{1.26 \pm 0.014}$ \\
GIN+F & $0.739 \pm 0.019$ & $75.4 \pm 0.9$ & $1.197 \pm 0.061$ & $76.5 \pm 2.0$ & $0.208 \pm 0.005$ \\
\homGNN{GIN}+F & $\mathbf{0.71 \pm 0.021}$ & $75.2 \pm 0.8$ & $\mathbf{1.014 \pm 0.044}$ & $\mathbf{77.7 \pm 1.5}$ & $\mathbf{0.174 \pm 0.005}$ \\
GCN+F & $1.188 \pm 1.387$ & $77.2 \pm 0.6$ & $1.197 \pm 0.069$ & $78.3 \pm 1.0$ & $0.234 \pm 0.007$ \\
\homGNN{GCN}+F & $\mathbf{0.816 \pm 0.282}$ & $\mathbf{78.0 \pm 0.6}$ & $\mathbf{0.991 \pm 0.045}$ & $\mathbf{78.8 \pm 1.3}$ & $\mathbf{0.207 \pm 0.008}$\\\bottomrule
\end{tabular}
\vskip 0.1in
\end{table}

\begin{table}[t]
\centering
\caption{Ablations of GIN with misaligned homomorphism counts (GIN+MIS+F) against GIN without homomorphism counts (GIN+F) and GIN with homomorphism counts(homGNN{GIN}+F). Results with (misaligned) homomorphism counts are averaged over 9 different samples of pattern graphs. \textbf{Bold} results are GIN with (misaligned) homomorphism counts that are better than GIN without homomorphism counts. \textcolor{red}{Red} results are \homGNN{GIN}+F that outperforms GIN+MIS+F.}
\label{tab:GNN_ablations}
\vskip 0.15in
\begin{tabular}{llllll}\toprule

\tablecell{ DATA SET $\rightarrow$ \\ $\downarrow$ MODEL} & \tablecell{MOLBACE \\ roc-auc $\uparrow$} & \tablecell{MOLCLINTOX \\ roc-auc $\uparrow$} & \tablecell{MOLBBBP\\ roc-auc$\uparrow$}  & \tablecell{MOLSIDER\\ roc-auc $\uparrow$} & \tablecell{MOLTOXCAST \\ roc-auc $\uparrow$} \\
\midrule
 & ogbg-molbace & ogbg-molclintox & ogbg-molbbbp & ogbg-molsider & ogbg-moltoxcast \\
GIN+F & $75.5 \pm 3.0$ & $84.8 \pm 3.7$ & $67.2 \pm 1.5$ & $57.7 \pm 1.8$ & $61.8 \pm 0.8$ \\
GIN+MIS+F & $\mathbf{76.3 \pm 3.3}$ & $\mathbf{86.0 \pm 3.6}$ & $\mathbf{67.9 \pm 2.3}$ & $\mathbf{58.3 \pm 1.9}$ & $\mathbf{62.7 \pm 0.9}$ \\
\homGNN{GIN}+F & \textcolor{red}{$\mathbf{76.4 \pm 2.6}$} & \textcolor{red}{$\mathbf{86.9 \pm 3.5}$} & \textcolor{red}{$\mathbf{68.8 \pm 1.3}$} & \textcolor{red}{$\mathbf{58.4 \pm 1.5}$} & \textcolor{red}{$\mathbf{63.2 \pm 0.8}$}  \\ \midrule
 &  \tablecell{MOLLIPO \\ rmse $\downarrow$} & \tablecell{MOLTOX21\\ roc-auc $\uparrow$} & \tablecell{MOLESOL\\ rsmse $\downarrow$} \\
 \midrule
 & ogbg-mollipo & ogbg-moltox21 & ogbg-molesol \\
GIN+F & $0.739 \pm 0.019$ & $75.4 \pm 0.9$ & $1.197 \pm 0.061$ \\
GIN+MIS+F & $\mathbf{0.732 \pm 0.017}$ & $74.9 \pm 0.9$ & $\mathbf{1.175 \pm 0.067}$ \\
\homGNN{GIN}+F & \textcolor{red}{$\mathbf{0.71 \pm 0.021}$} & \textcolor{red}{$75.2 \pm 0.8$} & \textcolor{red}{$\mathbf{1.014 \pm 0.044}$}\\\bottomrule
\end{tabular}
\vskip -0.1in
\end{table}

\end{document}